\documentclass[10pt,twocolumn,letterpaper]{article}
\usepackage{cvpr}
\usepackage{times}
\usepackage{epsfig}
\usepackage{graphicx}
\usepackage{amsmath}
\usepackage{amssymb}
\usepackage{amsfonts}
\usepackage[caption=false]{subfig}
\usepackage{algorithm}
\usepackage{algorithmic}
\usepackage{array}
\usepackage{multirow}
\usepackage{amsthm}

\newtheorem{thm}{Theorem}
\newtheorem{propo}{Proposition}

\newcolumntype{L}[1]{>{\raggedright\let\newline\\\arraybackslash\hspace{0pt}}m{#1}}
\newcommand{\R}{\ensuremath{\mathbb{R}}}

\newcommand*{\defeq}{\mathrel{\vcenter{\baselineskip0.5ex \lineskiplimit0pt
                     \hbox{\scriptsize.}\hbox{\scriptsize.}}}
                     =}

\usepackage{environ}
\NewEnviron{myequation}{%
\begin{equation}
\scalebox{0.95}{$\BODY$}
\end{equation}
}
\usepackage[pagebackref=true,breaklinks=true,letterpaper=true,colorlinks,bookmarks=false]{hyperref}

\cvprfinalcopy % *** Uncomment this line for the final submission

\def\cvprPaperID{2380} % *** Enter the CVPR Paper ID here

\ifcvprfinal\fi
\begin{document}

\title{Robust Camera Location Estimation by Convex Programming}

\author{Onur~\"{O}zye\c{s}il$^1$ and Amit Singer$^{1,2}$\\
$^1$Program in Applied and Computational Mathematics, Princeton University\\
$^2$Department of Mathematics, Princeton University\\
Princeton, NJ 08544-1000, USA\\
{\tt\small \{oozyesil,amits\}@math.princeton.edu}
}

\maketitle
\thispagestyle{empty} 
\pagestyle{empty} 

\begin{abstract}
$3$D structure recovery from a collection of $2$D images requires the estimation of the camera locations and orientations, i.e. the camera motion. For large, irregular collections of images, existing methods for the location estimation part, which can be formulated as the inverse problem of estimating $n$ locations $\mathbf{t}_1, \mathbf{t}_2, \ldots, \mathbf{t}_n$ in $\R^3$ from noisy measurements of a subset of the pairwise directions $\frac{\mathbf{t}_i - \mathbf{t}_j}{\|\mathbf{t}_i - \mathbf{t}_j\|}$, are sensitive to outliers in direction measurements. In this paper, we firstly provide a complete characterization of well-posed instances of the location estimation problem, by presenting its relation to the existing theory of parallel rigidity. For robust estimation of camera locations, we introduce a two-step approach, comprised of a pairwise direction estimation method robust to outliers in point correspondences between image pairs, and a convex program to maintain robustness to outlier directions. In the presence of partially corrupted measurements, we empirically demonstrate that our convex formulation can even recover the locations exactly. Lastly, we demonstrate the utility of our formulations through experiments on Internet photo collections.
\end{abstract}\vspace{-0.15in}
\section{Introduction}
Structure from motion (SfM) is the problem of recovering a $3$D (stationary) structure by estimating the camera motion corresponding to a collection of $2$D images of the same structure. Classically, SfM involves three steps: $(1)$ Estimation of point correspondences between pairs of images, and relative pose estimation of camera pairs based on corresponding points $(2)$ Estimation of camera motion, \ie global camera orientations and locations, from relative poses $(3)$ $3$D structure recovery based on the estimated motion by reprojection error minimization (\eg, using the bundle adjustment algorithm of~\cite{BundleAdjustment}). Although there exist accurate and efficient algorithms for the first and the third steps, existing methods for camera motion estimation, and specifically for the camera location estimation part, are usually sensitive to noise. The camera location estimation problem can be formulated as a specific case (for $d=3$) of the inverse problem of estimating $n$ locations $\mathbf{t}_1, \ldots, \mathbf{t}_n$ in $\R^d$ from a subset of (potentially noisy) measurements of the pairwise directions $\mathbf{\gamma}_{ij}$, given by 
\begin{equation}
\label{eq:GammaMats}
\mathbf{\gamma}_{ij} = \frac{\mathbf{t}_i - \mathbf{t}_j}{\|\mathbf{t}_i - \mathbf{t}_j\|} \\
\end{equation}
(see Figure~\ref{fig:LineEstInstance} for a noiseless instance of the problem). In terms of this formulation, misidentified point correspondences may manifest themselves (see \S\ref{sec:SubspaceEstim}) as direction measurements with large errors (\ie, outlier directions), and hence, may induce instability in location estimation.
\begin{figure}[!htbp]
\begin{center}\vspace{-0.08in}
   \includegraphics[trim=0cm 0cm 0cm 0cm, clip=true, width=0.9\linewidth]{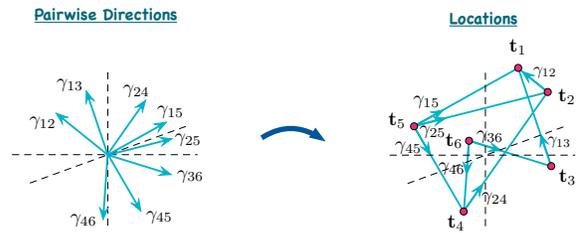}
\end{center}\vspace{-0.08in}
   \caption{A (noiseless) instance of the location estimation problem in $\R^3$, with $n=6$ locations and $m=8$ pairwise directions.\label{fig:LineEstInstance}}
\end{figure}
\newline\indent Existing methods for SfM can roughly be classified into two main categories; incremental approaches (\eg~\cite{SnavelyRome,VisualSfM,FurukawaPartial,HavlenaGroups,SnavelyData,SnavelySkeletal,ZhangIncremental}), that integrate images to the estimation process one by one (or in small groups) and global methods, that aim to estimate the camera motion (and sometimes also the $3$D structure) jointly for all images. Incremental methods are prone to accumulation of estimation errors at each step. On the other hand, for the global methods, since simultaneous estimation of motion and $3$D structure is computationally expensive, a usual procedure is to estimate motion and structure separately. Given an accurate motion estimate, a single instance of reprojection error minimization is usually enough to obtain high quality structure estimates.
\newline\indent Since orientation estimation is a relatively well-posed problem, with several efficient and stable existing methods (\eg~\cite{MicaAmitSfM,GovinduRot,GovinduLieAlg,HartleyRotations,MartinecRotations,CvXSfM}), it is customary to estimate the locations separately (based on the orientation estimates). The works of~\cite{MicaAmitSfM, BATL2, GovinduEarlyL2}, formulate the problem as finding a least squares solution to a linear system of equations derived from pairwise direction measurements (we refer to this method as the ``least squares'' (LS) solver). However, empirical observations (as in~\cite{CvXSfM}) have pointed out the instability of the LS solver, in the form of a tendency to produce spurious solutions clustering around a few locations. The multistage linear method of~\cite{MultiLinear} attempts to eliminate the clustering solutions by estimating the relative scales between cameras. The Lie algebraic averaging method of~\cite{GovinduLieAlg} is an efficient alternative, but may suffer from convergence to local minima. \cite{HartleySim} formulates a quasi-convex method (based on iterative optimization of a functional of the $\ell_\infty$ norm). However, since the $\ell_{\infty}$ norm is prone to outlier directions, this method usually fails to produce accurate estimates (see, \eg,~\cite{CvXSfM}). A relatively accurate method, closely related to our formulation, is studied in~\cite{TronVidalDist,TronVidalJournal}. Based on minimizing the $\ell_2$ norm of the error in direction measurements (linearized in $\mathbf{t}_i$'s), this method also employs constraints to eliminate clustering solutions (hence, we refer to this method as the ``constrained least squares'' (CLS) solver). However, in the presence of outlier directions, the accuracy of the CLS solver is degraded (see Figure~\ref{fig:SyntheticNRMSEsn200}). Another method closely related to our formulation, minimizing the $\ell_{\infty}$ norm of the error in direction measurements, is studied in~\cite{MoulonLinfty}. The accuracy of this method is affected by the sensitivity of the $\ell_{\infty}$ norm to outlier directions. In~\cite{JiangTriLin}, a global linear method, which uses triplets of images instead of pairwise directions, is studied. In the recent work~\cite{Snavely1D}, a preprocessing step (named 1DSfM, and designed to remove outlier directions), followed by a non-convex optimization method is introduced. Another recently introduced alternative is the ``semidefinite relaxation'' (SDR) solver of~\cite{CvXSfM}. Formulated as an abstract problem to estimate locations from pairwise ``lines'' (\ie, from measurements of $\pm\gamma_{ij}$, where the sign is unknown), this method aims to resolve the instability of the LS method by introducing extra non-convex constraints in the LS problem, and then relaxing them. However, semidefinite programming is computationally expensive for large data sets, and its accuracy is degraded in the presence of outlier lines (see Figure~\ref{fig:SyntheticNRMSEsn200}). 
\newline\indent In this paper, we characterize well-posed instances of the camera location estimation problem, by presenting its relation to the existing results of parallel rigidity theory. For robust estimation of camera locations, we introduce a two-step formulation: robust estimation of pairwise directions (in the presence of outliers in point correspondences), and a convex program for robust estimation of camera locations in the presence of measurements corrupted by large errors, \ie outlier directions. We provide empirical evaluation of our formulations using synthetic data, which demonstrate highly accurate location recovery performance compared to existing methods, and even exact location recovery in the presence of partially corrupted measurements (with sufficiently many noiseless directions). We also provide experimental results using real images, that present the accuracy and efficiency of our methods. 
\vspace{0.03in}\newline\noindent {\bf Notation:} We denote vectors in $\R^d$, $d\geq 2$, in boldface. For $\mathbf{x}\in\R^d$, $\|\mathbf{x}\|$ denotes its Euclidean norm. $S^d$ and $\mbox{SO}(d)$ denote the (Euclidean) sphere in $\R^{d+1}$ and the special orthogonal group of rotations acting on $\R^d$, respectively. We use the hat accent, to denote estimates of our variables, as in $\hat{X}$ is the estimate of $X$. We use star to denote solutions of optimization problems, as in $X^*$. Lastly, we use the letters $n$ and $m$ to denote the number of locations $|V_t|$ and the number of edges $|E_t|$ of graphs $G_t = (V_t,E_t)$ that encode the pairwise direction information.
\section{Location Estimation}
\label{sec:TransEstim}
The entire information of pairwise directions is represented using a measurement graph $G_t = (V_t, E_t)$, where the $i$'th node in $V_t = \{1,2,\ldots,n\}$ corresponds to the location $\mathbf{t}_i$ and each edge $(i,j)\in E_t$ is endowed with the direction $\mathbf{\gamma}_{ij}$. Provided with the set $\{\gamma_{ij}\}_{(i,j)\in E_t}$ of (noiseless) directions on $G_t = (V_t, E_t)$, we first study the problem of {\em unique realizability} of the locations. We will then introduce our robust formulation for location estimation from noisy pairwise directions.
\subsection{Parallel Rigidity}
\label{sec:ParRigidity}
The unique realizability of locations from (noiseless) pairwise directions was previously studied under the general title of {\em parallel rigidity theory} (see, \eg,~\cite{ErenNetwork,ErenNetwork2,KatzUnique,MaximalRigid,ServatiusWhiteleyCAD,WhiteleyMatroid,WhiteleyMatroidBook} and references therein). In the context of SfM, the implications of the parallel rigidity theory for the camera location estimation part were recognized in~\cite{CvXSfM}\footnote{We note that, although the pairwise measurements studied in~\cite{CvXSfM} are of the form $\pm\mathbf{\gamma}_{ij}$ (where, the sign is unavailable), the results of parallel rigidity theory for unique realizability remain the same when the signs are given.}. Here, we present a brief summary of fundamental results in parallel rigidity theory.
\newline\indent Provided with the noiseless pairwise directions $\{\mathbf{\gamma}_{ij}\}_{(i,j)\in E_t}\subseteq S^{d-1}$ (termed a ``formation''), we first consider the following fundamental questions: Can we uniquely realize $\{\mathbf{t}_i\}_{i\in V_t}$, of course, up to a global translation and scale (\ie can we obtain a set of points {\em congruent} to $\{\mathbf{t}_i\}_{i\in V_t}$)? Is unique realizability a generic property of the measurement graph $G_t$ (\ie is it independent of the particular realization of the points, assuming they are in generic position) and can it be decided efficiently? 
\begin{figure}[!htbp]
\begin{center}
   \includegraphics[trim=0cm 0cm 0cm 0cm, clip=true, width=0.95\linewidth]{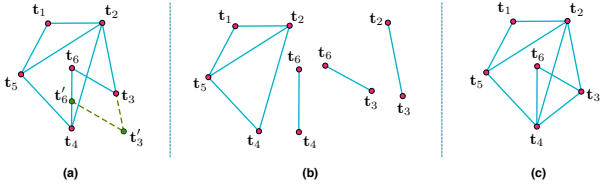}
\end{center}\vspace{-0.1in}
   \caption{{\bf(a)} A formation of $6$ locations on a connected graph, which is parallel rigid in $\R^3$ but {\bf not}  parallel rigid in $\R^2$. Non-uniqueness in $\R^2$ is demonstrated by two non-congruent location solutions $\{\mathbf{t}_1,\mathbf{t}_2,\mathbf{t}_3,\mathbf{t}_4,\mathbf{t}_5, \mathbf{t}_6\}$ and $\{\mathbf{t}_1,\mathbf{t}_2,\mathbf{t}_3',\mathbf{t}_4,\mathbf{t}_5,\mathbf{t}_6'\}$, each of which can be obtained from the other by an {\bf independent rescaling} of the solution for one of its maximally parallel rigid components in $\R^2$, {\bf(b)} Maximally parallel rigid components in $\R^2$, of the formation in {\rm(a)}, {\bf(c)} A parallel rigid formation (in $\R^2$ and $\R^3$) obtained from the formation in {\rm(a)} by adding the extra edge $(3,4)$ linking its maximally parallel rigid components\label{fig:ParallelRigidityEx}}
\end{figure}
\newline\indent Certifying unique realizability of locations is more complicated, \eg, compared to certifying uniqueness of camera orientations, which only requires (for arbitrary $d$) the connectivity of the measurement graph (see Figure~\ref{fig:ParallelRigidityEx}). On the other hand, parallel rigidity theory has a much simpler structure compared to the (classical) rigidity theory involving distance information (for a survey in rigidity theory, see~\cite{AspnesSurvey}). The identification of parallel rigid formations is addressed in~\cite{ErenNetwork,ErenNetwork2,WhiteleyMatroid,WhiteleyMatroidBook,KatzUnique} (also see the survey~\cite{JacksonJordanSurvey}), where it is shown that parallel rigidity in $\R^d$ ($d\geq2$) is a generic property of $G_t$ that admits a complete combinatorial characterization:
 \begin{thm}[Whiteley,1987]
\label{thm:LamanConds}
For a graph $G = (V,E)$, let $(d - 1)E$ denote the set consisting of $(d - 1)$ copies of each edge in $E$. Then,
$G$ is generically parallel rigid in $\R^d$ if and only if there exists a nonempty set $D\subseteq(d - 1)E$, with $|D| = d|V| - (d+1)$, such that for all subsets $D'$ of $D$,
\begin{equation}
\label{eq:LamanIneqs}
|D'| \leq d|V(D')| - (d+1) \ ,
\end{equation}
where $V(D')$ denotes the vertex set of the edges in $D'$.
\end{thm}
\indent The conditions of Theorem~\ref{thm:LamanConds} can be used to design efficient algorithms (\eg, adaptations of the pebble game algorithm~\cite{PebbleGame}, with a time complexity of $\mathcal{O}(n^2)$) for testing parallel rigidity. Also,~\cite{CvXSfM} provides a randomized spectral test (having a time complexity of $\mathcal{O}(m)$) for testing parallel rigidity. Moreover, unique realizability turns out to be equivalent to parallel rigidity, for arbitrary $d$ (see~\cite{ErenNetwork,JacksonJordanSurvey,KatzUnique,CvXSfM,WhiteleyMatroidBook}). 
\newline\indent For a formation that is not parallel rigid, the algorithms in~\cite{KatzUnique,MaximalRigid} can be used to decompose the graph into maximally parallel rigid components (\ie, to obtain maximal subgraphs of $G_t$ that can be uniquely realized).
\newline \indent The results of parallel rigidity theory are valid for {\em noiseless} directions. However, when provided with {\em noisy} directions (\eg, computed from real images), instead of uniqueness of the solution of a specific camera location estimation algorithm, we consider the following question: Is there sufficient information for the location estimation problem to be {\em well-posed} (in the sense that, if direction measurement error is small enough, then locations can be estimated stably)? For formations which are not parallel rigid, instability results from independent scaling and translation of maximally rigid components. Hence, we consider problem instances on parallel rigid measurement graphs to be {\em well-posed}. As a result, given a (noisy) formation $\{\mathbf{\gamma}_{ij}\}_{(i,j)\in E_t}$ on $G_t = (V_t,E_t)$, we firstly check for parallel rigidity of $G_t$, then, if the formation is not parallel rigid, we extract its maximally parallel rigid components (using the algorithm in~\cite{MaximalRigid}) and estimate the locations for the largest maximally parallel rigid component of $G_t$.
\subsection{Robust Location Estimation}
\label{sec:LocationEstim}
This section introduces our main formulation for robust location estimation. Suppose we are given a set of pairwise direction measurements $\{\mathbf{\gamma}_{ij}\}_{(i,j)\in E_t}\subseteq S^{d-1}$, \ie, for each $(i,j)\in E_t$, $\mathbf{\gamma}_{ij}$ satisfies
\begin{equation}
\label{eq:NoisyDirection}
\mathbf{\gamma}_{ij} = \frac{\mathbf{t}_i-\mathbf{t}_j }{\|\mathbf{t}_i-\mathbf{t}_j\|} + \epsilon_{ij}^{\gamma}
\end{equation}
where, $\epsilon_{ij}^{\gamma}$ denotes the direction error. Our objective is to estimate the locations $\{\mathbf{t}_i\}_{i\in V_t}$ (from the directions $\{\mathbf{\gamma}_{ij}\}_{(i,j) \in E_t}$) by maintaining  robustness to outlier direction measurements (\ie, $\mathbf{\gamma}_{ij}$'s with large $\epsilon_{ij}^{\gamma}$'s) in a computationally efficient manner. In this respect, we first rewrite (\ref{eq:NoisyDirection}) as
\begin{align}
\label{eq:NoisyDirectionLinT1}
\mathbf{t}_i-\mathbf{t}_j  &= \|\mathbf{t}_i-\mathbf{t}_j\| \mathbf{\gamma}_{ij}+ \epsilon_{ij}^{\mathbf{t}}\\
\label{eq:NoisyDirectionLinT2}
&= d_{ij}\mathbf{\gamma}_{ij}+ \epsilon_{ij}^{\mathbf{t}}\\
\label{eq:NoisyDirectionLinT3}
\iff\epsilon_{ij}^{\mathbf{t}} &= \mathbf{t}_i-\mathbf{t}_j - d_{ij}\mathbf{\gamma}_{ij}
\end{align}
where, $\epsilon_{ij}^{\mathbf{t}}$ denotes the displacement error, and we define $d_{ij}\defeq\|\mathbf{t}_i-\mathbf{t}_j\|$ to rewrite $\epsilon_{ij}^{\mathbf{t}}$ linearly in $\mathbf{t}_i$, $\mathbf{t}_j$ and $d_{ij}$. Observe that, large errors in directions (\ie, large $\epsilon_{ij}^{\gamma}$'s) induce large displacement errors $\epsilon_{ij}^{\mathbf{t}}$'s. As a result, we can employ displacement error minimization as a substitute for direction error minimization for location estimation. 
\newline\indent Hence, to maintain robustness to large $\epsilon_{ij}^{\mathbf{t}}$'s in~(\ref{eq:NoisyDirectionLinT3}), we choose to minimize the sum of {\em unsquared} norms of $\epsilon_{ij}^{\mathbf{t}}$'s. Also, for computational efficiency, we drop the intrinsic non-convex constraints $d_{ij} = \|\mathbf{t}_i-\mathbf{t}_j\|$ to obtain the convex ``least unsquared deviations'' (LUD) formulation
\begin{equation}
\label{eq:LUD}
\begin{aligned}
&\hspace{-0.25in}\underset{{\scriptstyle \begin{subarray}{c}
\{\mathbf{t}_i\}_{i\in V_t}\subseteq \R^d \\ 
\{d_{ij}\}_{(i,j)\in E_t}\end{subarray}}}{\text{minimize}}
& & \hspace{-0.1in} \sum_{(i,j)\in E_{t}}  \left\| \mathbf{t}_i-\mathbf{t}_j - d_{ij}\mathbf{\gamma}_{ij} \right\| \\
& \hspace{-0.18in} \text{subject to}
& & \sum_{i\in V_t} \mathbf{t}_i = \mathbf{0} \ ; \ d_{ij}\geq c, \ \forall (i,j) \in E_t 
\end{aligned}
\end{equation}
where the constraints $\sum_i \mathbf{t}_i = \mathbf{0}$ and $d_{ij}\geq c$ remove the translational and the scale ambiguities of the solution, respectively ({\em wlog} we take $c = 1$)\footnote{We note that the least squares version of~(\ref{eq:LUD}) (\ie, the program with the cost function $\sum_{(i,j)\in E_{t}}  \left\| \mathbf{t}_i-\mathbf{t}_j - d_{ij}\mathbf{\gamma}_{ij} \right\|^2$, and the same constraints as in~(\ref{eq:LUD})), which we name the ``constrained least squares'' (CLS) method, was previously studied in~\cite{TronVidalDist,TronVidalJournal}. However, as we experimentally demonstrate in \S\ref{sec:Experims}, the CLS formulation fails to maintain robustness to outliers. Also, the $\ell_{\infty}$ version of~(\ref{eq:LUD}), using the same constraints, was studied in~\cite{MoulonLinfty}.}. The constraints $d_{ij}\geq c$ are introduced to prevent trivial solutions of the form $d_{ij}^*\equiv0$, $\mathbf{t}_i^*\equiv\mathbf{0}$, as well as solutions clustered around a few locations. 
\newline\indent For well-posed instances of the location estimation problem (\ie, for parallel rigid $G_t$), and in the presence of noiseless direction measurements (\ie, $\epsilon_{ij}^{\gamma}\equiv\mathbf{0}$ in (\ref{eq:NoisyDirection})), we expect the LUD and CLS solvers to recover the locations $\mathbf{t}_i$ exactly. 
\begin{propo}[Exact Recovery in the Noiseless Case]
\label{propo:ExactRecovery}
Assume that the noiseless formation $\{\mathbf{\gamma}_{ij}\}_{(i,j)\in E_t}$, corresponding to the locations $\{\mathbf{t}_i\}_{i\in V_t}$ (in general position), is parallel rigid. Then, the LUD~(\ref{eq:LUD}) and CLS solvers recover the locations exactly, in the sense that any solution is congruent to $\{\mathbf{t}_i\}_{i\in V_t}$.
\end{propo}
\begin{proof}
{\em Wlog}, we assume $\min_{(i,j)\in E_t} \|\mathbf{t}_i-\mathbf{t}_j\| = 1$ and $\sum_i \mathbf{t}_i = \mathbf{0}$. Then, $\{\mathbf{t}_i\}_{i\in V_t}$ together with $d_{ij} = \|\mathbf{t}_i - \mathbf{t}_j\|$, $(i,j)\in E_t$, constitute an optimal solution for the LUD~(\ref{eq:LUD}) and CLS problems, with zero cost value. Let, $\{\mathbf{t}_i'\}_{i\in V_t}$ and $\{d_{ij}'\}_{(i,j)\in E_t}$ be another solution of the LUD~(\ref{eq:LUD}) and CLS problems, which must also have zero cost value. Then, for each $(i,j)\in E_t$, we get
\begin{align}
\nonumber
&\|\mathbf{t}_i' - \mathbf{t}_j'-d_{ij}'\mathbf{\gamma}_{ij}\| = 0 \iff  \mathbf{t}_i' - \mathbf{t}_j' = d_{ij}'\mathbf{\gamma}_{ij}\\
&\implies  \frac{\mathbf{t}_i' - \mathbf{t}_j'}{\|\mathbf{t}_i' - \mathbf{t}_j'\|} = \frac{\mathbf{t}_i - \mathbf{t}_j }{\|\mathbf{t}_i - \mathbf{t}_j\|}
\end{align}
\ie, $\{\mathbf{t}_i'\}_{i\in V_t}$ induces a formation on $G_t$, which is parallel to the formation corresponding to $\{\mathbf{t}_i\}_{i\in V_t}$ on $G_t$. However, since $G_t$ is parallel rigid, $\{\mathbf{t}_i'\}_{i\in V_t}$ has to be congruent to $\{\mathbf{t}_i\}_{i\in V_t}$ (in fact, $\mathbf{t}_i' = \alpha\mathbf{t}_i$, for $\alpha \geq 1$, by the feasibility of $\{\mathbf{t}_i',d_{ij}'\}$).
\qquad\end{proof}
\subsection{Iteratively Reweighted Least Squares (IRLS)}
\label{sec:IRLS}
In this section we formulate an iteratively reweighted least squares (IRLS) solver (see, \eg,~\cite{IRLS1,IRLS2}) for the LUD problem~(\ref{eq:LUD}). The main idea of IRLS is to iteratively solve (successive smooth regularizations of) the convex problem by using quadratic programing (QP) approximations. A pseudo code version is provided in Algorithm~\ref{alg:IRLS} (where, we consider a single smooth regularization). At the $r$'th iteration, more emphasis is given to the directions that are better approximated by the estimates $\hat{\mathbf{t}}_i^r$'s and $\hat{d}_{ij}^r$'s. Also, the regularization parameter $\delta$ ensures that no single direction can attain unbounded influence. The iterations are repeated until a convergent behavior in the variables, and in the cost value of the problem is observed. We refer the reader to~\cite{IRLS3} for a proof of convergence of the IRLS solver (where, a sequence of smooth regularizations, with $\delta\searrow0$, is assumed\footnote{Although we use a single smooth approximation by fixing $\delta\ll1$ for simplicity, we always obtained a convergent behavior in our experiments.}).
\begin{algorithm}[!htbp]
\small
\begin{algorithmic}
\STATE 
  {\bf Initialize:} $w_{ij}^0 = 1, \forall (i,j)\in E_t$  \vspace{0.05in}
\STATE  {\bf for } $r = 0,1,\ldots$ {\bf do} \vspace{0.05in}
\STATE \hspace{0.1in}$\left\lfloor 
\begin{aligned} &{\textstyle \mbox{\textbullet \ \ Compute } \{\hat{\mathbf{t}}_i^{r+1}\},\{\hat{d}_{ij}^{r+1}\} \mbox{ by solving the QP:}}\\[0.05in]
&\hspace{0.2in}{\displaystyle \underset{\left\{\begin{subarray}{c}
\sum\mathbf{t}_i = \mathbf{0},\\ 
d_{ij} \geq1\end{subarray}\right\}}{\text{minimize}} \sum_{(i,j)\in E_{t}}  w_{ij}^r\left\| \mathbf{t}_i-\mathbf{t}_j - d_{ij}\mathbf{\gamma}_{ij} \right\|^2}\\[0.02in]
&{\textstyle \mbox{\textbullet \ \ }w_{ij}^{r+1} \leftarrow \left( \left\| \hat{\mathbf{t}}_i^{r+1}-\hat{\mathbf{t}}_j^{r+1} - \hat{d}_{ij}^{r+1}\mathbf{\gamma}_{ij} \right\|^2+\delta \right)^{-1/2}} \end{aligned} 
\right.$
\end{algorithmic}
\caption{Iteratively reweighted least squares (IRLS) algorithm for the LUD~(\ref{eq:LUD}) solver\label{alg:IRLS}}
\end{algorithm}
\section{Robust Pairwise Direction Estimation}
\label{sec:SubspaceEstim}
We now present a pairwise direction estimation method designed to maintain robustness to outlier point correspondences between image pairs.
\newline\indent Let $\{\mbox{I}_i\}_{i=1}^n$ be a collection of images of a stationary $3$D scene. We use a pinhole camera model, and denote the orientations, locations, and focal lengths of the $n$ cameras corresponding to these images by $\{R_i\}_{i=1}^n \subseteq \mbox{SO}(3)$, $\{\mathbf{t}_i\}_{i=1}^n \subseteq \R^3$, and $\{f_i\}_{i=1}^n \subseteq \R^+$, respectively. Consider a scene point $\mathbf{P} \in \R^3$ represented in the $i$'th image plane by $\mathbf{p}_i \in \R^3$. To produce $\mathbf {p}_i$, $\mathbf {P}$ is firstly represented in the $i$'th camera's coordinate system by $\mathbf {P}_i = R_i^T(\mathbf{P} - \mathbf{t}_i) = (\mathbf{P}_i^x,\mathbf{P}_i^y,\mathbf{P}_i^z)^T$ and then projected onto the $i$'th image plane by $\mathbf{p}_i = (f_i/P_i^z)\mathbf{P}_i$. Note that, for the image $\mbox{I}_i$, we in fact observe $\mathbf{q}_i = (\mathbf{p}_i^x,\mathbf{p}_i^y)^T \in \R^2$ (\ie, the coordinates on the image plane) as the measurement corresponding to $\mathbf{P}$.
\newline \indent For an image pair $\mbox{I}_i$ and $\mbox{I}_j$, the essential matrix $E_{ij} = [\mathbf{t}_{ij}]_{\times}R_{ij}$ (where $R_{ij} = R_i^TR_j$ and $\mathbf{t}_{ij} = R_i^T(\mathbf{t}_j-\mathbf{t}_i)$ denote the pairwise rotation and translation, and $[\mathbf{t}_{ij}]_{\times}$ is the matrix of cross product with $\mathbf{t}_{ij}$) satisfies the ``epipolar constraints'' given by
\begin{align}
\label{eq:EpipolarConst}
&\mathbf{p}_i^TE_{ij}\mathbf{p}_j = 0 \\
\label{eq:EpipolarConst2}
\iff & {\small \left[\begin{matrix} \mathbf{q}_i/f_i \\ 1\end{matrix}\right]^TE_{ij}\left[\begin{matrix} \mathbf{q}_j/f_j \\ 1\end{matrix}\right] = 0}
\end{align}
\newline\indent The estimates $\hat{R}_{ij}$ and $\hat{\mathbf{t}}_{ij}$, computed from the decomposition of $\hat{E}_{ij}$ (estimated via (\ref{eq:EpipolarConst2})), usually have large errors due to misidentified and/or small number of corresponding points. Hence, instead of using existing algorithms (\eg,~\cite{MicaAmitSfM,HartleyRotations, MartinecRotations}) to estimate the orientations $\hat{R}_i$ and then computing the pairwise direction estimates $\hat{\gamma}_{ij} =  \hat{R}_i\hat{\mathbf{t}}_{ij}/\|\hat{\mathbf{t}}_{ij}\|$, we take the following approach (a similar approach is used in~\cite{CvXSfM}): First, the rotation estimates $\hat{R}_i$ are computed using the iterative method in \S{4.1} of~\cite{CvXSfM} (using the robust algorithm of~\cite{GovinduRot} for each iteration), and we then use the epipolar constraints~(\ref{eq:EpipolarConst2}) to robustly estimate the pairwise directions. 
\newline\indent To that end, we rewrite the epipolar constraint~(\ref{eq:EpipolarConst2}) to emphasize its linearity in $\mathbf{t}_i$ and $\mathbf{t}_j$. Let $\{\mathbf{q}_{i}^{k}\}_{k=1}^{m_{ij}}$ and $\{\mathbf{q}_{j}^{k}\}_{k=1}^{m_{ij}}$ denote $m_{ij}$ corresponding feature points. Then, for $\mathbf{\eta}_i^k = \left[\begin{smallmatrix} \mathbf{q}_i^k/f_i \\ 1\end{smallmatrix}\right]$ and $\mathbf{\eta}_j^k = \left[\begin{smallmatrix} \mathbf{q}_j^k/f_j \\ 1\end{smallmatrix}\right]$, we can rewrite~(\ref{eq:EpipolarConst2}) as (also see~\cite{MicaAmitSfM,KneipL2Dirs,CvXSfM})
\begin{align}
\nonumber
&(\mathbf{\eta}_i^k)^TE_{ij}\mathbf{\eta}_j^k =  \left(R_i\mathbf{\eta}_i^k\times R_j\mathbf{\eta}_j^k \right)^T\left(\mathbf{t}_i -\mathbf{t}_j\right) = 0\\
\nonumber
\iff  &(\mathbf{\nu}_{ij}^k)^T\left(\mathbf{t}_i - \mathbf{t}_j\right) = 0, \ \ \mbox{for } \ \nu_{ij}^k \ \mbox{ denoting}  \\ 
\label{eq:EpipolarRewrite}
&\mathbf{\nu}_{ij}^k \defeq \mathbf{\Theta} \left(R_i\mathbf{\eta}_i^k \times R_j\mathbf{\eta}_j^k\right)
\end{align}
where, the normalization function $\mathbf{\Theta}$ is defined by $\mathbf{\Theta}(\mathbf{x}) = \mathbf{x}/\|\mathbf{x}\|$, $\mathbf{\Theta}(\mathbf{0}) = \mathbf{0}$. Then, in the noiseless case (assuming $m_{ij}\geq2$, and that we can find at least two $\nu_{ij}^k$'s not parallel to each other), $\{\mathbf{\nu}_{ij}^k\}_{k=1}^{m_{ij}}$ determine a $2$D subspace orthogonal to $\mathbf{t}_i - \mathbf{t}_j$, and hence the (undirected) ``line'' through $\mathbf{t}_i$ and $\mathbf{t}_j$ (\ie, $\mathbf{\gamma}_{ij}^0 = b_{ij}\mathbf{\gamma}_{ij}$, where the sign $b_{ij} \in \{-1,+1\}$ is unknown, but can be determined by using the fact that the $3$D scene points should lie in front of the cameras).
\newline\indent In the presence of noisy measurements, \ie if we replace $R_i$'s, $f_i$'s and $\mathbf{q}_i$'s with their estimates in~(\ref{eq:EpipolarRewrite}), we essentially obtain noisy samples $\hat{\mathbf{\nu}}_{ij}^k$'s from the $2$D subspace orthogonal to $\mathbf{t}_{i}-\mathbf{t}_{j}$. In order to maintain robustness to outliers among $\hat{\mathbf{\nu}}_{ij}^k$'s in the estimation of (undirected) lines $\mathbf{\gamma}_{ij}^0$, we first consider the following (non-convex) problem:
\begin{equation}
\label{eq:REAPER}
\begin{aligned}
\underset{{\scriptstyle \gamma_{ij}^0}}{\text{minimize}}
& \ \ \sum_{k=1}^{m_{ij}}  | (\gamma_{ij}^0)^T\hat{\mathbf{\nu}}_{ij}^k| \\
\text{subject to} & \ \ \|\gamma_{ij}^0\| = 1 \,
\end{aligned}
\end{equation}
In order to obtain the estimate $\hat{\mathbf{\gamma}}_{ij}^0$, we use a (heuristic) IRLS method for (\ref{eq:REAPER}). Here, although the program (\ref{eq:REAPER}) is not convex, and hence the IRLS method is not guaranteed to converge to global optima, we empirically observed this approach to produce high quality estimates for the lines $\mathbf{\gamma}_{ij}^0$, while preserving computational efficiency (for alternative methods, see~\cite{SSReaper,SSTylerM,SSGiannakis}). Lastly, the estimates $\hat{b}_{ij}$ of the signs of the direction estimates $\hat{\gamma}_{ij} = \hat{b}_{ij}\hat{\gamma}_{ij}^0$ are computed using the fact that the $3$D points should lie in front of the cameras.
\newline\indent In Figure~\ref{fig:DirEstimCompare}, we provide a comparison of our robust direction estimation method, with a PCA-based estimator (comprised of solving (\ref{eq:REAPER}) by replacing the cost function with the sum of squares version, \ie with $\sum_{k}  | (\gamma_{ij}^0)^T\hat{\mathbf{\nu}}_{ij}^k|^2$). The results imply that, the accuracy of the direction estimates can be significantly improved by our robust method. We also note that the running time of our robust method is comparable to that of PCA and hence does not significantly increase the overall running time of the entire pipeline.
\begin{figure}[!htbp]
\centering
   \includegraphics[trim=0.7cm 1cm 0.6cm 1.5cm, clip=true, width=1\linewidth]{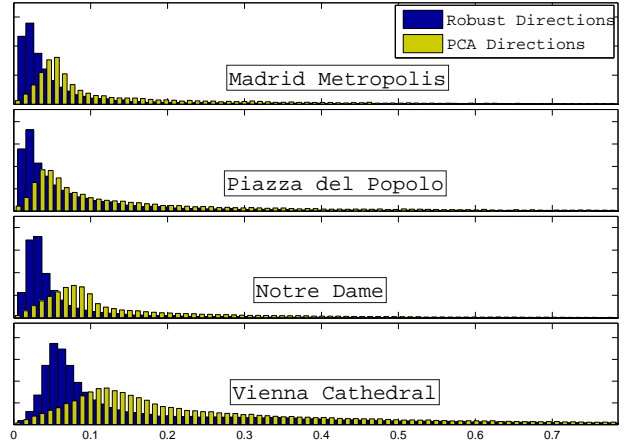}
\vspace{-0.08in}
   \caption{Histogram plots of the errors in direction estimates computed by our robust method (\ref{eq:REAPER}) and the PCA method, for some of the datasets (from~\cite{Snavely1D}) studied in \S\ref{sec:RealDataExp}. The errors represent the angles between the estimated directions and the corresponding ground truth directions (computed from a sequential SfM method based on~\cite{SnavelyData}, and provided in~\cite{Snavely1D}). We note that the errors take values in $[0,\pi]$, yet the histograms are restricted to $[0,\pi/4]$ to emphasize the difference of the quality in the estimated directions.\label{fig:DirEstimCompare}}
\end{figure}\\
\indent A summary of our camera motion estimation algorithm is given in Table~\ref{tab:AlgorithmSummary}.
\begin{table}[!htbp]
\centering
\resizebox{0.9\columnwidth}{!}{
\tabcolsep=0.1cm
\renewcommand{\arraystretch}{1.8}
\begin{tabular}{L{2cm}||L{9cm}}
\hline 
& Input: Images: $\{\mbox{I}_i\}_{i=1}^n$, \ Focal lengths: $\{f_i\}_{i=1}^n$ \vspace{0.035in}\\ \hline\hline
Feature Points, Essential Matrices, Camera Orientations & {\bf 1.} Find corresponding points between images (using SIFT~\cite{MicaAmit22}) \newline {\bf 2.} Compute $\hat{E}_{ij}$, using the eight-point algorithm~\cite{HartleyBook} (for pairs with sufficiently many correspondences) \newline {\bf 3.} Factorize $\hat{E}_{ij}$ to compute $\{\hat{R}_{ij}\}_{(i,j)\in E_{R}}$ and $G_R = (V_R,E_R)$ 
\newline {\bf 4.} Compute the orientation estimates $\hat{R}_{i}$ from $\{\hat{R}_{ij}\}_{(i,j)\in E_{R}}$ \newline \hspace{0.13in}(using the iterative approach of~\cite{CvXSfM} with the robust method \newline \hspace{0.13in}of~\cite{GovinduRot} for each iteration)
\vspace{0.035in} \\ \hline
Robust Pairwise Direction Estimation \S\ref{sec:SubspaceEstim}& {\bf 5.} Compute the $2$D subspace samples $\{\hat{\nu}_{ij}^k\}_{k=1}^{m_{ij}}$ for each \newline \hspace{0.13in}$(i,j)\in E_R$ (\ref{eq:EpipolarRewrite}) \newline {\bf 6.} Estimate the pairwise directions $\{\hat{\gamma}_{ij}\}_{(i,j)\in E_{R}}$ using (\ref{eq:REAPER})  \vspace{0.035in}\\ \hline
Location Estimation \S\ref{sec:TransEstim}& {\bf 7.} Extract the largest maximally parallel rigid component \newline \hspace{0.13in}$G_t = (V_t,E_t)$ of $G_R$ (see~\cite{MaximalRigid}) \newline {\bf 8.} Compute the location estimates $\{\hat{\mathbf{t}}_i\}_{i\in V_t}$ by \newline \hspace{0.13in}the LUD (\ref{eq:LUD}) method (using the IRLS Algorithm~\ref{alg:IRLS} or classical \newline \hspace{0.13in}interior point methods, \eg,~\cite{SDPT32}) \vspace{0.035in}\\ \hline\hline
 & Output: Camera orientations and translations: $\{\hat{R}_i, \hat{\mathbf{t}}_i\}$ \vspace{0.035in}\\
\hline
\end{tabular}}
\vspace{0.05in}\caption{Algorithm for camera motion estimation\label{tab:AlgorithmSummary}}
\end{table}
\section{Experiments} 
\label{sec:Experims}
\subsection{Synthetic Data Experiments}
\label{sec:SyntExp}
In this section we provide synthetic data experiments for the LUD formulation~(\ref{eq:LUD}). In particular, we provide evidence for exact location recovery from partially corrupted directions, and also compare the LUD solver to the CLS~\cite{TronVidalDist,TronVidalJournal}, the SDR~\cite{CvXSfM} and the LS~\cite{MicaAmitSfM,BATL2} methods.\\
\indent The measurement graphs $G_t = (V_t,E_t)$ of our experiments are random graphs drawn from the Erd\H{o}s-R\'{e}nyi model $\mathcal{G}(n,q)$, \ie each $(i,j)$ is in the edge set $E_t$ with probability $q$, independently of all other edges. In each experiment, we only record the results of problem instances defined on parallel rigid $G_t$. Given a set of locations $\{\mathbf{t}_i\}_{i=1}^n \subseteq \R^d$ and $G_t = (V_t,E_t)$, for each $(i,j)\in E_t$, we first let
\begin{equation}
\tilde{\gamma}_{ij} = \begin{cases} \mathbf{\gamma}_{ij}^U \ \ , & \mbox{w.p. } \ p \\ (\mathbf{t}_i-\mathbf{t}_j)/\|\mathbf{t}_i-\mathbf{t}_j\| + \sigma\mathbf{\gamma}_{ij}^G \, & \mbox{w.p.} \ 1-p \end{cases}
\label{eq:NoiseModel}
\end{equation}
and normalize $\tilde{\gamma}_{ij}$'s to obtain $\mathbf{\gamma}_{ij} = \tilde{\gamma}_{ij}/\|\tilde{\gamma}_{ij}\|$ as the direction measurement for the pair $(i,j)$. Here, $\{\mathbf{\gamma}_{ij}^U\}_{(i,j)\in E_t}$  and $\{\mathbf{\gamma}_{ij}^G\}_{(i,j)\in E_t}$ are i.i.d. random variables drawn from the uniform distribution on $S^{d-1}$ and the standard normal distribution on $\mathbb{R}^d$, respectively. Also, the original locations $\mathbf{t}_i$'s are i.i.d. random variables drawn from standard normal distribution on $\mathbb{R}^d$.\\
\indent We evaluate the performance in terms of the ``normalized root mean squared error'' (NRMSE) given by
\begin{equation}
\label{eq:NRMSE}
\mbox{NRMSE}(\{\hat{\mathbf{t}}_i\}) = \sqrt{\frac{\sum_i \|\hat{\mathbf{t}}_i - \mathbf{t}_i\|^2} {\sum_i \|\mathbf{t}_i - \mathbf{t}_0\|^2}}
\end{equation}
where $\hat{\mathbf{t}}_i$'s are the location estimates (after removal of the global scale and translation) and $\mathbf{t}_0$ is the center of $\mathbf{t}_i$'s. \\
\indent The first set of experiments demonstrates the recovery performance of the LUD solver in the presence of partially corrupted directions, by setting $\sigma = 0$ in~(\ref{eq:NoiseModel}), and by controlling the proportion of outlier measurements via the parameter $p$. The results are summarized in Figure~\ref{fig:ExactRecoveryTest}, where for each experiment the intensity of each pixel represents $\log_{10}(\mbox{NRMSE})$ (NRMSE values are averaged over $10$ random realizations). These results demonstrate a striking feature of the LUD solver: In the presence of partially corrupted directions (with sufficiently small, but non-zero, proportion of corrupted directions), the LUD solver {\em recovers the original locations exactly} (\ie, we get $\mbox{NRMSE} < \epsilon_{\mbox{\scriptsize IRLS}}$, where $\epsilon_{\mbox{\scriptsize IRLS}}$ is the convergence tolerance for the IRLS algorithm, set to $\epsilon_{\mbox{\scriptsize IRLS}} = 1\mbox{e-}8$ in our experiments). In Figure~\ref{fig:ExactRecoveryTest}, we observe that, the exact recovery performance for $d = 3$ is  improved as compared to the $d = 2$ case. Additionally, the transition to the exact recovery region becomes slightly sharper, and exact recovery performance for small values of outlier probability $p$ is marginally improved when enlarging $n$ from $100$ to $200$.
\begin{figure}[!htbp]
\begin{center}
\includegraphics[trim=2.3cm 7.8cm 0.8cm 1.5cm, clip=true, width=0.85\linewidth]{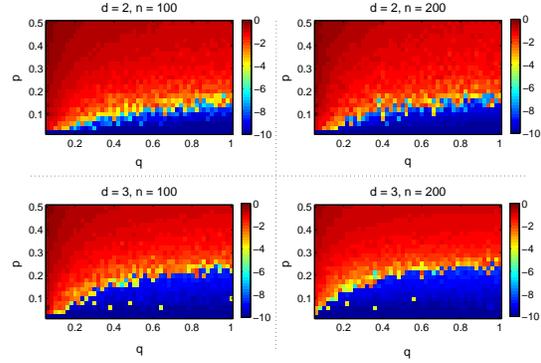}
\end{center}\vspace{-0.2in}
 \caption{NRMSE~(\ref{eq:NRMSE}) results of the LUD~(\ref{eq:LUD}) solver for the exact recovery experiments. The color intensity of each pixel represents $\log_{10}(\mbox{NRMSE})$, depending on the edge probability $q$ ($x$-axis), and the outlier probability $p$ ($y$-axis). Measurements are generated by the noise model~$(\ref{eq:NoiseModel})$, assuming $\sigma = 0$, and NRMSE values are averaged over $10$ trials.\label{fig:ExactRecoveryTest}}
\end{figure}\\
\indent The second set of experiments, depicted in Figure~\ref{fig:SyntheticNRMSEsn200}, presents a comparative evaluation of the NRMSE of the LUD, the CLS, the SDR and the LS solvers, for $d=3$ (we observed similar performance for $d=2$). The outcomes clearly present the robustness of the LUD formulation in the presence of outliers (up to a significant proportion of outliers, depending on $q$ and $n$), while the recovery performance of the other methods is degraded significantly. Even if the measurement noise is dominated by small errors in the inlier directions (\ie, when $\sigma$ is relatively large compared to $p$), the LUD solver continues to outperform the other methods, in almost all cases.
\begin{figure}[!htbp]
\begin{center}\vspace{-0.05in}
   \includegraphics[trim=1.8cm 1cm 1.8cm 0.5cm, clip=true, width=0.85\linewidth]{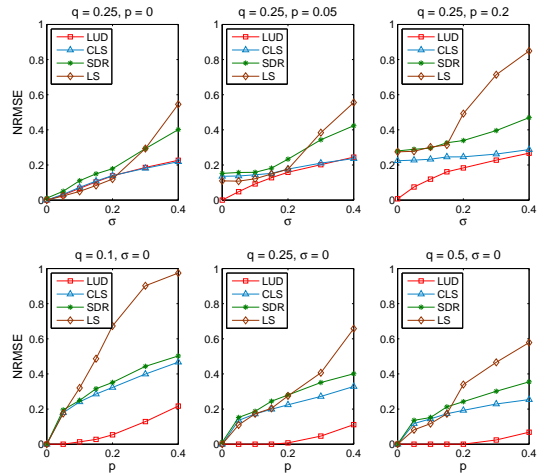}
\end{center}\vspace{-0.12in}
   \caption{NRMSE~$(\ref{eq:NRMSE})$ performance of the LUD~(\ref{eq:LUD}) formulation vs. the CLS~\cite{TronVidalDist,TronVidalJournal}, the SDR~\cite{CvXSfM} and the LS~\cite{MicaAmitSfM,BATL2} solvers, for $n = 200$ locations. Measurements are generated by the noise model~$(\ref{eq:NoiseModel})$ and NRMSE values are averaged over $10$ trials.\label{fig:SyntheticNRMSEsn200}}
\end{figure}
\begin{table*}[!htbp]
\centering \hspace{-0.1in}
\resizebox{2.1\columnwidth}{!}{
\renewcommand{\arraystretch}{1.05}
\tabcolsep=0.1cm
\begin{tabular}{|l|c|c||c|c|c|c|c|c|c||c|c|c|c|c||c|c|c|c|c||c|c|c|c||c|c|}
\hline
\multicolumn{3}{|c||}{Dataset} & \multicolumn{7}{c||}{LUD} & \multicolumn{5}{c||}{CLS~\cite{TronVidalDist,TronVidalJournal}} & \multicolumn{5}{c||}{SDR~\cite{CvXSfM}} & \multicolumn{4}{c||}{1DSfM~\cite{Snavely1D}} & \multicolumn{2}{c|}{\cite{GovinduEarlyL2}}\\ \hline
\multirow{3}{*}{Name} & \multicolumn{2}{c||}{Size} & \multicolumn{4}{c|}{Initial} & \multicolumn{3}{c||}{After BA} & \multicolumn{2}{c|}{Initial} & \multicolumn{3}{c||}{After BA} & \multicolumn{2}{c|}{Initial} & \multicolumn{3}{c||}{After BA} & \multicolumn{1}{c|}{Init.} & \multicolumn{3}{c||}{After BA} & \multicolumn{2}{c|}{After BA} \\ &\multicolumn{2}{c||}{}& \multicolumn{2}{c|}{PCA} &\multicolumn{2}{c|}{Robust}&\multicolumn{3}{c||}{Robust}&\multicolumn{2}{c|}{Robust}&\multicolumn{3}{c||}{Robust}&\multicolumn{2}{c|}{Robust}& \multicolumn{3}{c||}{Robust} & &\multicolumn{3}{c||}{} & \multicolumn{2}{c|}{} \\ 
& $m$ & $N_c$ & {\large$\tilde{e}$} & {\large $\hat{e}$} & {\large$\tilde{e}$} & {\large $\hat{e}$} & $N_c$ & {\large $\tilde{e}$} & {\large$\hat{e}$} & {\large$\tilde{e}$} & {\large$\hat{e}$} & $N_c$ & {\large$\tilde{e}$} & {\large$\hat{e}$} & {\large$\tilde{e}$} & {\large$\hat{e}$} & $N_c$ & {\large$\tilde{e}$} & {\large$\hat{e}$} & {\large$\tilde{e}$} & $N_c$ & {\large$\tilde{e}$} & {\large$\hat{e}$} & $N_c$ & {\large$\tilde{e}$}\\ \hline
Piazza del Popolo & $60$ & $328$ & $3.0$ & $7$ & $\mathbf{1.5}$ & $\mathbf{5}$ & $305$ & $\mathbf{1.0}$ & $\mathbf{4}$ & $3.5$ & $6$ & $305$ & $1.4$ & $5$ & $1.9$ & $8$ & $305$ & $1.3$ & $7$ & $3.1$ & $308$ & $2.2$ & $200$ & $93$ & $16$ \\ 
NYC Library & $130$ & $332$ & $4.9$ & $9$ & $\mathbf{2.0}$ & $\mathbf{6}$ & 320 & $1.4$ & $7$ & $5.0$ & $8$ & $320$ & $3.9$ & $8$ & $5.0$ & $8$ & $320$ & $4.6$ & $8$ & $2.5$ & $295$ & $\mathbf{0.4}$ & $\mathbf{1}$ & $271$ & $1.4$ \\ 
Metropolis & $200$ & $341$ & $4.3$ & $8$ & $\mathbf{1.6}$ & $\mathbf{4}$ & $288$ & $1.5$ & $\mathbf{4}$ & $6.4$ & $10$ & $288$ & $3.1$ & $7$ & $4.2$ & $8$ & $288$ & $3.1$ & $7$ & $9.9$ & $291$ & $\mathbf{0.5}$ & $70$ & $240$ & $18$ \\ 
Yorkminster & $150$ & $437$ & $5.4$ & $10$ & $\mathbf{2.7}$ & $\mathbf{5}$ & $404$ & $1.3$ & $\mathbf{4}$ & $6.2$ & $9$ & $404$ & $2.9$ & $8$ & $5.0$ & $10$ & $404$ & $4.0$ & $10$ & $3.4$ & $401$ & $\mathbf{0.1}$ & $500$ & $345$ & $6.7$ \\ 
Tower of London & $300$ & $572$ & $12$ & $25$ & $\mathbf{4.7}$ & $\mathbf{20}$ & $425$ & $3.3$ & $\mathbf{10}$ & $16$ & $30$ & $425$ & $15$ & $30$ & $20$ & $30$ & $425$ & $17$ & $30$ & $11$ & $414$ & $\mathbf{1.0}$ & $40$ & $306$ & $44$ \\ 
Montreal N. D. & $30$ & $450$ & $1.4$ & $2$ & $\mathbf{0.5}$ & $\mathbf{1}$ & $435$ & $\mathbf{0.4}$ & $\mathbf{1}$ & $1.1$ & $2$ & $435$ & $0.5$ & $\mathbf{1}$ & $-$ & $-$ & $-$ & $-$ & $-$ & $2.5$ & $427$ & $\mathbf{0.4}$ & $\mathbf{1}$ & $357$ & $9.8$\\ 
Notre Dame & $300$ & $553$ & $1.1$ & $2$ & $\mathbf{0.3}$ & $\mathbf{0.8}$ & $536$ & $\mathbf{0.2}$ & $\mathbf{0.7}$ & $0.8$ & $2$ & $536$ & $0.3$ & $0.9$ & $-$ & $-$ &$-$ & $-$ & $-$ & $10$ & $507$ & $1.9$ & $7$ & $473$ & $2.1$ \\ 
Alamo & $70$ & $577$ & $1.5$ & $3$ & $\mathbf{0.4}$ & $\mathbf{2}$ & $547$ & $\mathbf{0.3}$ & $\mathbf{2}$ & $1.3$ & $3$& $547$ & $0.6$ & $\mathbf{2}$ & $-$ & $-$ & $-$ & $-$ & $-$ & $1.1$ & $529$ & $\mathbf{0.3}$ & $2$e$7$ & $422$ & $2.4$\\ 
Vienna Cathedral & $120$ & $836$ & $7.2$ & $12$ & $\mathbf{5.4}$ & $\mathbf{10}$ & $750$ & $4.4$ & $\mathbf{10}$ & $8.8$ & $\mathbf{10}$ & $750$ & $8.2$ & $\mathbf{10}$ & $-$ & $-$ & $-$ & $-$ & $-$& $6.6$ & $770$ & $\mathbf{0.4}$ & $2$e$4$ & $652$ & $12$ \\ \hline
\end{tabular}
}\vspace{0.05in}
\caption{Performance comparison of various methods for datasets from~\cite{Snavely1D}: Units are (approximately) in meters. $N_c$ denotes number of estimated camera locations, $\hat{e}$ denotes the average distance, and $\tilde{e}$ denotes the median distance of the estimated camera locations to the corresponding cameras in the reference solution (computed using~\cite{SnavelyData}, and provided in~\cite{Snavely1D}). `PCA' and `Robust' refers to the pairwise direction estimation method used (\cf (\ref{eq:REAPER}) and Figure~\ref{fig:DirEstimCompare}).\label{tab:RealDataExpErrs}}
\end{table*}
\subsection{Real Data Experiments}
\label{sec:RealDataExp}
We tested our location estimation algorithm on nine sets of real images from~\cite{Snavely1D}. These are relatively irregular collections of images and hence estimating the camera locations for all of these images (or a large subset) is challenging. To solve the LUD problem~(\ref{eq:LUD}), we use the IRLS algorithm~\ref{alg:IRLS}, and to construct a $3$D structure in our experiments, we use the parallel bundle adjustment (PBA) algorithm of~\cite{PBA}. We perform our computations on workstations with Intel(R) Xeon(R) X$7542$ CPUs, each with $6$ cores, running at $2.67$ GHz. In order to directly compare the accuracy of the location estimation by LUD to that of CLS~\cite{TronVidalDist,TronVidalJournal} and SDR~\cite{CvXSfM} solvers, we use the same direction estimates (\cf Table~\ref{tab:AlgorithmSummary}) for each method (except for the case where the PCA directions are used for the LUD solver, \cf columns $4$ and $5$ of Table~\ref{tab:RealDataExpErrs}). These estimates produced more accurate location estimates for all data sets. We note that, the computation of the robust direction estimates is performed in parallel (using $10$ cores for each dataset). Similar to~\cite{Snavely1D}, for performance evaluation, we consider the camera location estimates computed by a sequential SfM solver based on Bundler~\cite{SnavelyData} (and provided in~\cite{Snavely1D}) as the ground truth, and use a RANSAC-based method to compute the global transformation between our estimates and the ground truth.\\
\indent We provide the accuracy comparisons in Table~\ref{tab:RealDataExpErrs}: The results are given in terms of the average distance $\hat{e}$, and the median distance $\tilde{e}$ of the estimated camera locations to the corresponding cameras in the reference solution (units are approximately in meters). The results of~\cite{Snavely1D} correspond to the estimates computed by the combination of an outlier direction detection method (termed ``1DSfM'' in~\cite{Snavely1D}) and a location estimation method employing a robust cost function. The results of~\cite{GovinduEarlyL2} are cited from~\cite{Snavely1D}. Also, the results of the SDR method~\cite{CvXSfM} correspond to the estimates computed by applying the solver to the whole measurement graphs, and hence are not provided for the relatively larger datasets due to computational limitations. We also provide the running times corresponding to each experiment in Table~\ref{tab:RealDataExpTimes} (note that the bundle adjustment times $T_{BA}$ for the LUD, the CLS and the SDR solvers are computed after an initial $3$D structure is provided). The comparison of the accuracy of the LUD solver given the robust directions (\cf\S\ref{sec:SubspaceEstim}) to the case of the PCA directions, and the comparison of the LUD solver to the CLS, the SDR and~\cite{Snavely1D} imply that, the combination of our robust direction estimation method and the LUD solver produces highly accurate initial estimates, with a computation cost that is slightly higher than the CLS method and~\cite{Snavely1D}. Using the initial estimates, we apply PBA once, to obtain rich $3$D structures and further improvements in accuracy. See Figure~\ref{fig:3DNewDataSets} for some of the $3$D structures obtained from the initial LUD estimates.

\begin{table*}[!htbp]
\centering \hspace{-0.1in}
\resizebox{2.1\columnwidth}{!}{
\renewcommand{\arraystretch}{1.05}
\tabcolsep=0.1cm
\begin{tabular}{|l|c|c|c||c|c|c||c|c|c||c|c|c||c|c|c|c|c||c||c|}
\hline
\multicolumn{4}{|c||}{} & \multicolumn{3}{c||}{LUD} & \multicolumn{3}{c||}{CLS~\cite{TronVidalDist,TronVidalJournal}} & \multicolumn{3}{c||}{SDR~\cite{CvXSfM}} & \multicolumn{5}{c||}{1DSfM~\cite{Snavely1D}} & \cite{GovinduEarlyL2} & \cite{SnavelyData}\\ 
Dataset & {\large $T_{R}$} & {\large $T_G$} & {\large$T_{\gamma}$} & {\large$T_{\mathbf{t}}$} & {\large$T_{BA}$} & {\large$T_{tot}$} & {\large$T_{\mathbf{t}}$} & {\large$T_{BA}$} & {\large$T_{tot}$} & {\large$T_{\mathbf{t}}$} & {\large$T_{BA}$} & {\large$T_{tot}$} & {\large$T_{R}$} & {\large$T_{\gamma}$} & {\large$T_{\mathbf{t}}$} & {\large$T_{BA}$} & {\large$T_{tot}$} & {\large$T_{tot}$} & {\large$T_{tot}$} \\ \hline
Piazza del Popolo & $35$ & $43$ & $18$ & $35$ & $31$ & $162$ & $9$ & $106$ & $211$ & $358$ & $39$ & $493$ & $14$ & $9$ & $35$ & $191$ & $249$ & $138$ & $1287$ \\ 
NYC Library & $27$ & $44$ & $18$ & $57$ & $54$ & $200$ & $7$ & $47$ & $143$ & $462$ & $52$ & $603$ & $9$ & $13$ & $54$ & $392$ & $468$ & $220$ & $3807$ \\ 
Metropolis & $27$ & $37$ & $13$ & $27$ & $38$ & $142$ & $6$ & $23$ & $106$ & $181$ & $45$ & $303$ & $15$ & $8$ & $20$ & $201$ & $244$ & $139$ & $1315$ \\ 
Yorkminster & $19$ & $46$ & $33$ & $51$ & $148$ & $297$ & $10$ & $133$ & $241$ & $648$ & $75$ & $821$ & $11$ & $18$ & $93$ & $777$ & $899$ & $394$ & $3225$ \\ 
Tower of London & $24$ & $54$ & $23$ & $41$ & $86$ & $228$ & $8$ & $202$ & $311$ & $352$ & $170$ & $623$ & $9$ & $14$ & $55$ & $606$ & $648$ & $264$ & $1900$ \\ 
Montreal N. D. & $68$ & $115$ & $91$ & $112$ & $167$ & $553$ & $21$ & $270$ & $565$ & $-$ & $-$ & $-$ & $17$ & $22$ & $75$ & $1135$ & $1249$ & $424$ & $2710$\\ 
Notre Dame & $135$ & $214$ & $325$ & $247$ & $126$ & $1047$ & $52$ & $504$ & $1230$ & $-$ & $-$ & $-$ & $53$ & $42$ & $59$ & $1445$ & $1599$ & $1193$ & $6154$ \\ 
Alamo & $103$ & $232$ & $96$ & $186$ & $133$ & $750$ & $40$ & $339$ & $810$ & $-$ & $-$ & $-$ & $56$ & $29$ & $73$ & $752$ & $910$ & $1403$ & $1654$\\ 
Vienna Cathedral & $267$ & $472$ & $265$ & $255$ & $208$ & $1467$ & $46$ & $182$ & $1232$ & $-$ & $-$ & $-$ & $98$ & $60$ & $144$ & $2837$ & $3139$ & $2273$ & $10276$ \\ \hline
\end{tabular}
}\vspace{0.05in}
\caption{Running times, in seconds, for the experiments in Table~\ref{tab:RealDataExpErrs}: times for orientation estimation ($T_{R}$), extraction of largest maximally parallel rigid component ($T_G$), robust pairwise direction estimation ($T_{\gamma}$), translation estimation ($T_{\mathbf{t}}$), bundle adjustment ($T_{BA}$), and total time ($T_{tot}$). (For the LUD, the CLS and the SDR solvers, the bundle adjustment times $T_{BA}$ are computed after an initial $3$D structure is provided, and the first three columns, \ie, $T_{R}$,$T_G$,$T_{\gamma}$, are common).\label{tab:RealDataExpTimes}}
\end{table*}

\begin{figure*}
\begin{center}
   \includegraphics[trim=0cm 0.25cm 0cm 0.1cm, clip=true, width=0.95\linewidth]{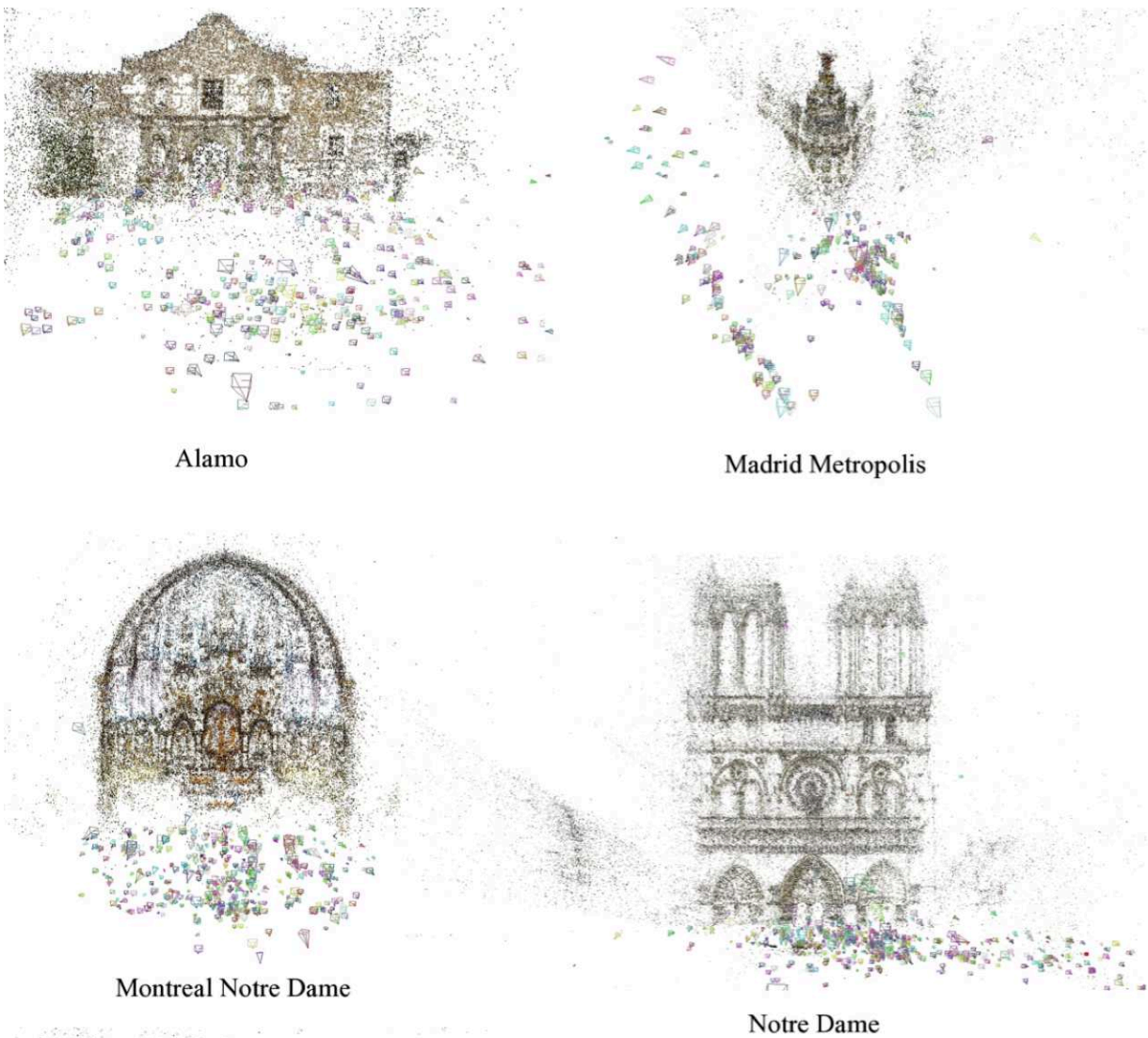}
\end{center}\vspace{-0.1in}
   \caption{Snapshots of selected $3$D structures computed using the camera location estimates of the LUD solver~(\ref{eq:LUD}) (without bundle adjustment). Each $3$D point is visible through at least three cameras.\label{fig:3DNewDataSets}}
\end{figure*}

\section{Conclusion and Future Work}
\label{sec:Conclusion}
We provided a complete characterization of well-posed instances of the camera location estimation problem, via the existing theory of parallel rigidity, and used it in practice to extract maximal image subsets for which estimation of camera location is well posed. For robust estimation of camera locations, we introduced a pairwise direction estimation method to maintain robustness to outliers in point correspondences, and we also presented a robust convex program, namely ``the least unsquared deviations'' (LUD) solver, to diminish the effects of outliers in pairwise direction measurements. We empirically demonstrated that unlike other estimators, the LUD formulation allows exact recovery of locations in the presence of partially corrupted direction measurements. In the context of structure from motion, our formulations can be used to efficiently and robustly estimate camera locations, in order to produce a high-quality initialization for reprojection error minimization algorithms, as demonstrated by our experiments on real image sets. 
\newline \indent As future work, we plan to further investigate the phenomenon of exact recovery with partially corrupted directions, to characterize the conditions for its existence.

\section*{Acknowledgements} The authors wish to thank Ronen Basri and Yuehaw Khoo for many valuable discussions related to this work. 
\newline \indent The authors were partially supported by  Award Number FA9550-12-1-0317 and FA9550-13-1-0076 from AFOSR, and by Award Number LTR DTD 06-05-2012 from the Simons Foundation.
{\small
\bibliographystyle{ieee}
\bibliography{SfMbib}
}

\end{document}